\documentclass[letterpaper, 10 pt, journal, twoside]{IEEEtran}  
\IEEEoverridecommandlockouts 

\usepackage{amsthm,amsmath,amsfonts}
\usepackage{algorithmic}
\usepackage{algorithm}
\usepackage{array}
\usepackage[caption=false,font=normalsize,labelfont=sf,textfont=sf]{subfig}
\usepackage{textcomp}
\usepackage{stfloats}
\usepackage{url}
\usepackage{verbatim}
\usepackage{graphicx}
\usepackage{cite}
\usepackage{caption}
\usepackage{array}
\usepackage[table]{xcolor}
\usepackage[font=scriptsize,labelfont=bf]{caption}
\usepackage{hyperref}

\usepackage{todonotes}
\usepackage{booktabs}
\usepackage{wrapfig}
\let\labelindent\relax
\usepackage{enumitem}
\usepackage[normalem]{ulem}

\newtheorem{theorem}{Theorem}

\newcommand{\edit}[1]{\textcolor{black}{#1}}

\newcommand{\pose}{\ensuremath{x}}
\newcommand{\poses}{\ensuremath{\mathcal{X}}}
\newcommand{\terrainFunc}{\ensuremath{T}}
\newcommand{\terrainFuncSet}{\ensuremath{\mathbb{T}}}
\newcommand{\terrains}{\ensuremath{\mathcal{T}}}
\newcommand{\humanFunc}{\ensuremath{H}}
\newcommand{\humanFuncSet}{\ensuremath{\mathcal{H}}}
\newcommand{\observationFunc}{\ensuremath{O}}
\newcommand{\images}{\ensuremath{\mathcal{I}}}
\newcommand{\image}{\ensuremath{I}}

\newcommand{\prefFunction}{\ensuremath{R}}
\newcommand{\pacerFunction}{\ensuremath{\hat{R}}}
\newcommand{\costmaps}{\ensuremath{\mathcal{C}}}
\newcommand{\costmap}{\ensuremath{C}}

\begin{document}

\title{\textsc{pacer}: Preference-conditioned All-terrain Costmap Generation}

\author{Luisa Mao$^{1}$, Garrett Warnell$^{1,2}$, Peter Stone$^{1,3}$, Joydeep Biswas$^{1,4}$

\thanks{Manuscript received: October, 19, 2024; Revised January, 21, 2025; Accepted Feburary, 16, 2025.}
\thanks{This paper was recommended for publication by Editor Cesar Cadena upon evaluation of the Associate Editor and Reviewers' comments.} 

\thanks{\scriptsize $^{1}$Luisa Mao, Garrett Warnell, Peter Stone, and Joydeep Biswas are with the Department of Computer Science, The University of Texas at Austin, Austin, TX, USA.
        {\tt\scriptsize luisa.mao@utexas.edu}}%
\thanks{\scriptsize $^{2}$Garrett Warnell is also with DEVCOM Army Research Laboratory, Austin, TX, USA.
        {\tt\scriptsize garrett.a.warnell.civ@army.mil}}%
\thanks{\scriptsize $^{3}$Peter Stone is also with Sony AI,  Boston, MA 02129, USA.
        {\tt\scriptsize pstone@cs.utexas.edu}}%
\thanks{\scriptsize $^{4}$Joydeep Biswas is also with NVIDIA, Santa Clara, CA 95051, USA.
        {\tt\scriptsize joydeepb@cs.utexas.edu}}%
\thanks{Digital Object Identifier (DOI): see top of this page.}
}

\markboth{IEEE Robotics and Automation Letters. Preprint Version. Accepted March, 2025}
{Mao \MakeLowercase{\textit{et al.}}: PACER}

\maketitle
\begin{abstract}

In autonomous robot navigation, terrain cost assignment is typically performed using a semantics-based paradigm in which terrain is first labeled using a pre-trained semantic classifier and costs are then assigned according to a user-defined mapping between label and cost. While this approach is rapidly adaptable to changing user preferences, only preferences over the types of terrain that are already known by the semantic classifier can be expressed. In this paper, we hypothesize that a machine-learning-based alternative to the semantics-based paradigm above will allow for rapid cost assignment adaptation to preferences expressed over {\em new} terrains at deployment time without the need for additional training. To investigate this hypothesis, we introduce and study \textsc{pacer}, a novel approach to costmap generation that accepts as input a single birds-eye view (BEV) image of the surrounding area along with a user-specified {\em preference context} and generates a corresponding BEV costmap that aligns with the preference context. Using a staged training procedure leveraging real and synthetic data, we find that \textsc{pacer} is able to adapt to new user preferences at deployment time while also exhibiting better generalization to novel terrains compared to both semantics-based and representation-learning approaches. We release our code and dataset at \footnotesize{\href{https://github.com/ut-amrl/PACER_RAL_2025.git}{https://github.com/ut-amrl/PACER\_RAL\_2025.git}}

\end{abstract}


\section{Introduction and Related Work}
\label{sec:introduction}
    \IEEEPARstart{R}{obust} autonomous navigation in a wide variety of environments is a long-standing goal in robotics. While there has been significant progress in collision-free navigation~\cite{xiao2022autonomous,xiao2023autonomous}, successful navigation in human environments additionally requires alignment with human preferences, e.g., preferring to cross a busy street at a crosswalk even if doing so results in a longer path~\cite{karnan2022socially,Meng-RSS-23}.

\edit{In this paper, we examine how robots can assign terrain costs that align with human preferences for terrain-aware navigation. This specific focus represents a special case of the broader challenge of human preference-aligned navigation~\cite{choi2020fast}. Such alignment is crucial not only for terrain-aware navigation, but also for adherence to constraints like social norms~\cite{8462900,wulfmeier2016incorporating}.
An alternative to learning to predict navigation \emph{costs} is to directly learn navigation \emph{policies} aligned with human 
 preferences~\cite{kahn2021badgr,kahn2021land,raj2024rethinking}. We choose to focus on predicting navigation costs due to the ease of integration with existing cost-based navigation planners~\cite{580977,7487277,Macenski_2023}. 
While the focus of this work is on terrain-based cost evaluation, we recognize that a full navigation system would incorporate additional costs using established methods such as layered costmaps~\cite{lu2014layered}}.
    

    We are particularly interested in terrain cost assignment approaches that can rapidly adapt to newly-expressed terrain preferences.  \edit{Prevalent approaches to incorporate human preference into navigation} such as inverse reinforcement learning (IRL) and preference-based IRL (PbIRL) based on terrain patch clusters do not admit this type of rapid adaptation due to the amount of additional data required to express new preferences~\cite{karnan2023sterling,sikand2022visual,karnan2024wait}. \edit{Instead, to the best of our knowledge, existing solutions for rapid adaptation to preferences rely on first segmenting terrains into a prescribed set of classes, and assigning each class a manually-specified cost~\cite{Meng-RSS-23,jiang2021rellis,guan2022ga}. While such approaches allows rapid adaptation to new preferences, they are restricted to expressing preferences over the pre-defined list of terrain classes known to the segmentation algorithm.}

    \begin{figure}
        \centering
        \includegraphics[width=\linewidth]{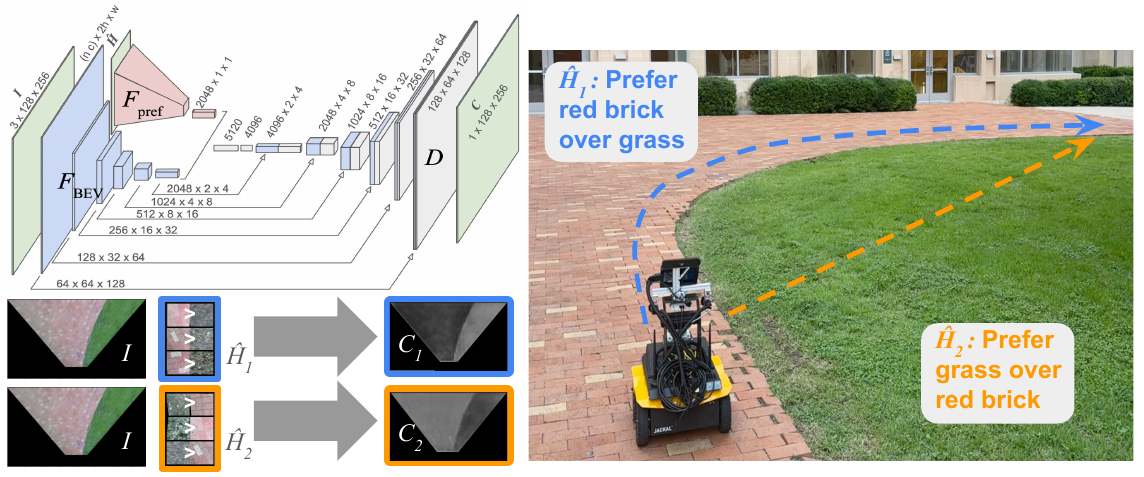}
        \caption{
        Given an input image \image{} and a preference context $\hat{\humanFunc{}}$ of $n$ ordered pairs of terrain patches where the left terrain is more preferred than the right, \textsc{pacer} generates a costmap consistent with this preference. Changing the preference context leads to changed terrain costs, which results in a different plan aligned to the new operator preference.
        The paths planned according to the different preferences are shown above. In the costmap, black represents low cost and white represents high cost.
        \label{fig:plan}}
        \vspace{-.5cm}
    \end{figure}

    \edit{
    \emph{Representation learning} is a more recent approach to terrain-aware navigation that allows preferences to no longer be limited to terrains with predefined labels \cite{karnan2023sterling, yao2022rca, zurn2020self, sikand2022visual}. Patch-based representation learning methods for terrain understanding typically involve mapping small square patches from the bird’s-eye view (BEV) to a representation vector that is further converted into a scalar cost value. Although continuous representation spaces are theoretically generalizable to new terrains, training such a space effectively is challenging in practice, as even humans may struggle to identify terrain types from small patches in the presence of homography artifacts or difficult lighting conditions. A further limitation is that each new terrain preference ordering necessitates retraining the utility function in the representation space, making these approaches less adaptable to changing preferences.}

    Towards overcoming the limitations of the semantics-based and representation-learning paradigms to terrain cost assignment, we propose and study \textsc{pacer}, a novel approach to costmap generation that accepts as input a single birds-eye view (BEV) image of the surrounding terrain along with a user-specified preference context and generates a corresponding BEV cost map that aligns with that preference context (see Fig. \ref{fig:plan}). By {\em preference context}, we mean a small set of terrain patches and pairwise preferences over those patches that are supplied at deployment time. We design \textsc{pacer} to exhibit three design desiderata: {\em (1)} it is capable of representing a prior over terrain preferences; {\em (2)} it is capable of adapting to a wide variety of preference contexts; and {\em (3)} it is able to assign aligned costs to terrains that appear in both the preference context and the BEV image, even for novel terrain types.

    Using real and synthetic terrain data, we implement a training pipeline to realize these three properties and evaluate the resulting preference-conditioned costmap functions over a wide variety of BEV images. Additionally, we study the impact of the resulting costmaps on cost-optimal navigation behavior with respect to adherence to human preferences. We find that our method overcomes limitations in prior works by being easily adaptable to new operator preferences and producing fine-grained costmaps that illicit desirable navigation behaviors even in previously unseen environments.

\section{The Terrain-Aware Preference-Aligned Planning Problem}
\label{sec:problem_formulation}

\begin{figure}
    \vspace{.3cm}
    \centering
        \includegraphics[width=\linewidth]{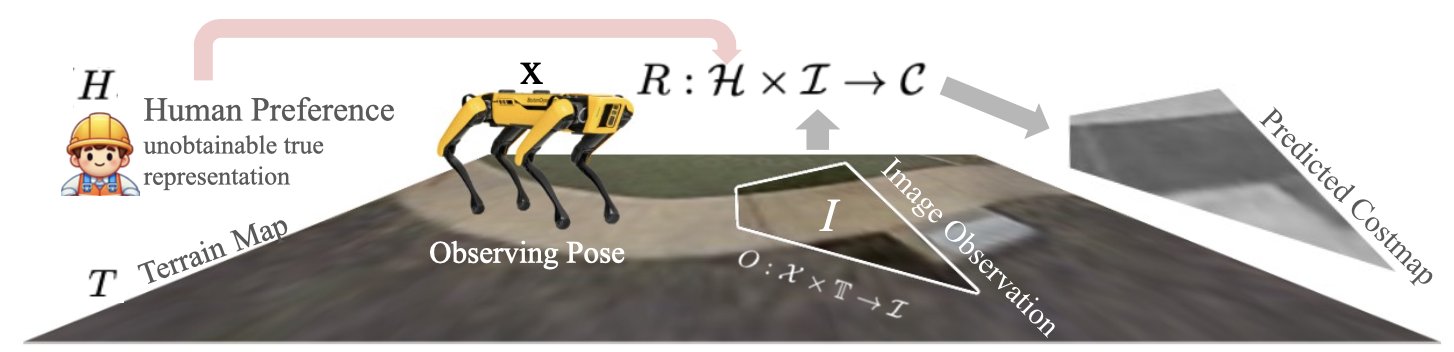}
        \caption{Relationships between spaces of Terrains, Image Observations, and Costmaps. There exists a hidden ``true" costing function based on human preferences directly on terrains. \textsc{pacer} approximates this function from visual observations of terrains.
        }
        \label{fig:function_diagram}
        \vspace{-.6cm}

\end{figure}
We now develop the terrain-aware preference-aligned planning problem. We will first formulate the path planning problem, and then we will discuss the problem of learning preference-aligned terrain costs.

\subsection{Path Planning}

In this paper, we are concerned with the general problem of planning a path in a robot state space $\poses{}$ ($\mathrm{SE}(2)$ for ground vehicles) from a start and goal pose $\pose{}_1, G \in \poses{}$ as the problem of finding the finite trajectory $\Gamma_S = [\pose{}_1, ..., \pose_S]$ consisting of $S$ states $\pose{} \in \poses{}$ which minimizes a total objective function
\begin{align}
    \Gamma_S = \arg_{\Gamma} \min ||\pose{}_S-G|| + \lambda \mathcal{J}(\Gamma),
\end{align}
where $||\pose{}_S-G||$ is the distance between the final state $\pose{}_S$ and $G$, and $\mathcal{J}(\Gamma)$ is the cost function scaled by the relative weight $\lambda$.

A cost function $\mathcal{J}(\Gamma)$ may include various terms such as the geometric cost of obstacles, social navigation cost, or terrain cost,
\begin{align}
    \mathcal{J}(\Gamma) = \mathcal{J}_{\text{geometric}}(\Gamma) + \alpha \mathcal{J}_{\text{social}}(\Gamma) + \beta \mathcal{J}_{\text{terrain}}(\Gamma)
\end{align}
where $\alpha, \beta$ are relative weights. This paper is concerned with the terrain cost term $\mathcal{J}_{\text{terrain}}(\Gamma)$ of the general function.

\subsection{Preference-Aligned Terrain Costs}
To better understand preference-aligned terrain costs, we first introduce a fixed terrain function to represent the spatial distribution of the terrains in the world.
Let a terrain map \(\terrainFunc: \poses{} \rightarrow \terrains{}\) be a function that maps a robot pose $\pose \in \poses$ to the terrain $\tau \in \terrains$ that the robot interacts with when in pose \pose{}, where a terrain $\tau$ captures all the properties of the ground relevant to robot navigation.

Additionally, we assume the human has an unknown true cost function $\humanFunc{}:\terrains{}\rightarrow \mathbb{R}^{0+}$ mapping terrains to scalar real-valued costs based on their preferences.
This cost function is influenced by various factors, including the personal preferences of human operator, the environment, and the task at hand.
Let \humanFuncSet{} denote the continuous space of such cost functions, such that \(\humanFunc{} \in \humanFuncSet{}\).

For terrain-aware navigation, the robot relies on its visual observations to infer terrain-specific costs.
We assume that these observations arrive in the form of images generated according to a black-box observation function $\observationFunc{}: \poses{} \times \terrainFuncSet{} \to \images{}$, i.e., $I = O(x, T)$, where $x$ is the observing pose of the robot, $T$ is a terrain map, \terrainFuncSet{} is the space of terrain maps, and \images{} is the space of images.
In practice, most methods operate on synthetic birds-eye-views generated from the original camera images. BEV images can be generated via static ground-plane homography \cite{sikand2022visual, karnan2023sterling}, or a BEV accumulation algorithm \cite{miki2022elevation}. Henceforth, we define input images to be BEV images.
We assume that the the visual appearance of the terrain provides sufficient information for the robot to perform terrain-aware navigation. The observation function is thus fixed, but unknown to the robot.

During planning, the terrain cost of a pose is found using a costmap $\costmap{}:\poses{}\to \mathcal{R}^{0+}$ that maps from robot poses to costs.
We introduce a costmap generation function $\prefFunction{}: \images{} \times \humanFuncSet{} \to \costmaps{}$ as the function
mapping from the space of images \images{} to the space of costmaps \costmaps{}, conditioned on an unknown human cost function that belongs to
\humanFuncSet{}.

Since the robot has no direct access to the terrain map $\terrainFunc{}$ and there is no clear representation of $\humanFunc{}$, the terrain-aware preference-aligned planning problem is thus to learn the function \prefFunction{} such that, given an image observation of terrain, the optimal trajectory planned with respect to \prefFunction{} is also optimal with respect \humanFunc{}. The conditions in the next section will be introduced as our analyses of how we address this problem.

\section{Necessary Conditions for Preference-Aligned Navigation}
\label{sec:necessary_conditions}
Seeking training tasks to help us compute valid preference-conditioned costmap functions $R(\cdot | H)$, we now state a set of necessary conditions for these tasks to produce costmaps that are consistent with human preferences for terrain.

In particular, we will state conditions for \emph{equivalence} and \emph{partial ordering}, and we will show that $R$s that produce costmaps that yield optimal trajectories consistent with a human preference must obey these conditions.

Let \(\prefFunction{}(\cdot | \humanFunc{})\) denote a costmap generated according to an $\humanFunc{} \in \humanFuncSet{}$, and $\costmap{}|_\pose{}$ denote that costmap \(\costmap{} \in \costmaps{} \) is evaluated at pose \pose{}. For a generated costmap \(\prefFunction{}(\cdot | \humanFunc{})\) to be consistent with $\humanFunc{}$, we specify it must exhibit both \emph{equivalence} and \emph{partial ordering}.
By equivalence, we mean that the terrains at two poses are given the same cost by \humanFunc{} if and only if the costmap generated by \prefFunction{} from an image observation and evaluated at those two poses have equal cost, i.e.,
\begin{align}
    \humanFunc{}_i(\terrainFunc(\pose_1)) &= \humanFunc{}_i(\terrainFunc(\pose_2)) \iff \prefFunction{}(\observationFunc{}(\cdot, \terrainFunc)\mid \humanFunc{}_i) \Big|_{\pose_1} \notag \\
    &= \prefFunction{}(\observationFunc{}(\cdot, \terrainFunc) \mid \humanFunc{}_i) \Big|_{\pose_2}  \quad \forall \pose_1, \pose_2, \humanFunc{}_i  \tag{NC1} \; . \label{eq:first}
\end{align}
By partial ordering, we mean that \humanFunc{} assigns a preference order over the terrains at two poses if and only if the costmap generated by \prefFunction{} from an image observation assigns those two poses the same preference order, i.e.,
\begin{align}
    \humanFunc{}_i(\terrainFunc(\pose_1)) &< \humanFunc{}_i(\terrainFunc(\pose_2)) \iff \prefFunction{}(\observationFunc{}(\cdot, \terrainFunc) \mid \humanFunc{}_i) \Big|_{\pose_1} \notag \\
    &< \prefFunction{}(\observationFunc{}(\cdot, \terrainFunc) \mid \humanFunc{}_i) \Big|_{\pose_2} \quad \forall \pose_1, \pose_2, \humanFunc{}_i \tag{NC2} \; . \label{eq:second}
\end{align}

We let $O(\cdot,T)$ denote an image observation captured from any observing pose from which $x_1, x_2$ are visible.

We now provide a brief proof that \eqref{eq:first}, \eqref{eq:second} are necessary for aligning the preferences of $\humanFunc{}_i$ with $\prefFunction{}$. Specifically, if the most optimal path with respect to $\prefFunction{}(\cdot \mid \humanFunc{}_i)$ has the same optimal cost when evaluated with $\humanFunc{}_i$, then the conditions \eqref{eq:first},\eqref{eq:second} must hold. For a trajectory $\Gamma$ composed of discrete poses, let the cumulative cost function for the human's evaluation be denoted by $\humanFunc{}|_{\Gamma} \equiv \sum_{\pose_i \in \Gamma} \humanFunc{} (T(\pose{}_i))$. Similarly, let the cumulative cost function for the generated costmap be denoted as $\prefFunction{} | _{\Gamma} \equiv \sum_{\pose_i \in \Gamma} \prefFunction{} (\observationFunc{}(\cdot, T) \mid \humanFunc{}) | _{\pose_i}$. Given this setup, the following theorem establishes the necessity of the conditions \eqref{eq:first}, \eqref{eq:second} such that $\humanFunc{} | _{\Gamma^\ast} = \prefFunction{} | _{\overline{\Gamma}}$.

\begin{theorem}
Let \(\Gamma ^ \ast = \arg_{\Gamma}  \min\humanFunc{}|_{\Gamma}, \quad \overline{\Gamma} = \arg_{\Gamma}  \min \prefFunction{} | _{\Gamma} \)
denote the optimal trajectories with respect to $\humanFunc{}$ and $\prefFunction{}$ respectively.
If the optimal trajectory with respect to \prefFunction{} has equal cost to the optimal path with respect to \humanFunc{} when both are evaluated on \humanFunc{} such that $\humanFunc{} | _{\Gamma^\ast} = \prefFunction{} | _{\overline{\Gamma}}$,
then conditions \eqref{eq:first} and \eqref{eq:second} hold.

\end{theorem}
\begin{proof}
    Since \( \humanFunc{} | _{\Gamma^\ast} = \prefFunction{} | _{\overline{\Gamma}}\), we must have that:
\begin{enumerate}[label=(\alph*)]
    \item \(\humanFunc{} | _{\Gamma_1} < \humanFunc{} | _{\Gamma_2} \Rightarrow  \prefFunction{} | _{\Gamma_1} <  \prefFunction{} | _{\Gamma_2}\) for all paths $\Gamma_1, \Gamma_2$.

    Otherwise, there exist paths $\Gamma_1, \Gamma_2$ such that \(  \humanFunc{} | _{\Gamma_1} <  \humanFunc{} | _{\Gamma_2}\) and \( \prefFunction{} | _{\Gamma_1} \geq  \prefFunction{} | _{\Gamma_2} \). Then, $\Gamma_2$ may be selected as $\overline{\Gamma}$, but has greater cost than $\humanFunc{} | _{\Gamma^\ast}$ when evaluated on \humanFunc{}, which is a contradiction.

    \item \( \prefFunction{} | _{\Gamma_1} <  \prefFunction{} | _{\Gamma_2} \Rightarrow  \humanFunc{} | _{\Gamma_1} <  \humanFunc{} | _{\Gamma_2}\) for all paths $\Gamma_1, \Gamma_2$.

    Otherwise, there exist paths $\Gamma_1, \Gamma_2$ such that \( \prefFunction{} | _{\Gamma_1} <  \prefFunction{} | _{\Gamma_2} \) and \(  \humanFunc{} | _{\Gamma_1} >  \humanFunc{} | _{\Gamma_2}\) (by contraposition on (a), we eliminate the case where \( \prefFunction{} | _{\Gamma_1} <  \prefFunction{} | _{\Gamma_2} \) and \(  \humanFunc{} | _{\Gamma_1} =  \humanFunc{} | _{\Gamma_2}\)). Then, $\Gamma_1$ may be selected as $\overline{\Gamma}$, but may have greater cost than $\humanFunc{} | _{\Gamma^\ast}$ when evaluated on \humanFunc{}, which is a contradiction.

\end{enumerate}

By (a) and (b), we have that \( \humanFunc{} | _{\Gamma_1} <  \humanFunc{} | _{\Gamma_2} \iff  \prefFunction{} | _{\Gamma_1} <  \prefFunction{} | _{\Gamma_2}\).
By contraposition, we also have \( \humanFunc{} | _{\Gamma_1} = \humanFunc{} | _{\Gamma_2} \iff  \prefFunction{} | _{\Gamma_1} = \prefFunction{} | _{\Gamma_2}\).
Finally, since a path $\Gamma$ can consist of a single state, conditions \eqref{eq:first} and \eqref{eq:second} must also then hold.
\end{proof}

In the next section, we use conditions (\ref{eq:first}) and (\ref{eq:second}) to define training tasks for learning the optimal $R$ from data, which drives our proposed approach to the online generation of costmaps which result in preference-aligned navigation.

\section{Preference-Aligned All-Terrain Costmap Generation}
We now present our proposed approach for computing aligned terrain costmaps, which we refer to as Preference-aligned, All-terrain Costmap genERation (\textsc{pacer}).
\textsc{pacer} introduces the notion of a preference context and comprises several components, including a neural network architecture, and a data curation and training methodology based on the three design desiderata.
    
\subsection{Preference Context}
 \label{sec:preference_context}
The preference-aligned terrain costs discussed in Section \ref{sec:problem_formulation} depend on a human's cost function $H: \mathcal{T} \rightarrow \mathbb{R}^{0+}$.
Unfortunately, we do not have access to $H$ directly since it is known only to the human operator.
Therefore, we propose to obtain and utilize an approximate representation of $H$ that we call a \emph{preference context}.

We define a preference context $\hat{\humanFunc{}}$ as a set of $n$ image patch pairs $\widetilde{\image{}}\succ \widetilde{\image{}}'$ constructed from human input 

such that the human prefers the terrain observed in image $\widetilde{\image{}}$ over the terrain observed in $\widetilde{\image{}}'$, where an \(\widetilde{\image{}} \in \images{}\) is an observation of terrain as a small image patch. The small patch may be a part of a larger bird's-eye-view image of the ground. More specifically, $\hat{\humanFunc{}}$ consists of $n$ preferences derived from \(\humanFunc{}\) and is defined as \(\hat{\humanFunc{}} \equiv \{(\widetilde{\image{}}_{1} \succ \widetilde{\image{}}_{1}'), ...(\widetilde{\image{}}_{n} \succ \widetilde{\image{}}_{n}')\}\).
Fig \ref{fig:dataset} shows some example preference contexts with $n=3$ patch pairs and their corresponding costmaps.

In our implementation, the $n$ pairwise preferences are expressed using image patches of size $h \times w$.
$\hat{\humanFunc{}}$ is then represented by vertically concatenating the patches within a pair with the more-preferred terrain patch on top and forming a single $(n \cdot c) \times 2h \times w$ tensor, where $c$ is the number of color channels. \edit{Given a finite preference context, it is impossible to specify all pairwise preferences over the terrains, especially since the terrain set is continuous - hence it is not possible to specify a stronger sufficient condition that would guarantee that the algorithm generates human-aligned costs.}

\subsection {Model Architecture}
To generate costmaps, we propose to approximate functions $R: \mathcal{I} \times \mathcal{H} \rightarrow \mathcal{C}$, which require $H$ as input, with functions $\hat{R}: \mathcal{I} \times \hat{\mathcal{H}} \rightarrow \mathcal{C}$, where $\hat{\mathcal{H}}$ is the space of all preference contexts as defined above.
We model $\hat{R}$ as a neural network with a two encoders and a single decoder.
The input image is passed through a BEV image encoder $F_{\mathrm{BEV}}$ to form an image embedding, and, similarly, the input preference context is passed through a preference context encoder $F_{\mathrm{pref}}$ to form a preference embedding.
The output costmap is then generated by concatenating these embeddings and then passing them through a decoder $D$.
A visual depiction of this architecture is provided in Fig. \ref{fig:plan}.


\vspace{0.25cm}
\subsection{Loss Function}
\textsc{pacer} is trained using supervised machine learning, i.e., given a dataset $\mathcal{D} = \{(\hat{H}, \image{},\costmap{}_T)_i\}_{i=1}^N$ of preference context, image, and target costmap tuples, we seek the parameters $\phi$ of $\hat{R}$ that minimize a loss between the real and predicted costmaps.
More specifically, we seek $\phi^\ast$ such that
\begin{align}
    \phi^\ast = \arg\min_\phi \mathbb{E}_{(I,\hat{H},C_T) \sim \mathcal{D}}\left[ \ell\left(\hat{R}_\phi(I,\hat{H}),C_T\right) \right] \; ,
\end{align}
where we use the binary cross entropy loss averaged over each pixel as the loss function $\ell$.

\section{Dataset Curation and Training \textsc{pacer}}
    \begin{figure*}[htb] 
        \vspace{0.3cm}
        \centering
        \includegraphics[width=0.8\linewidth]{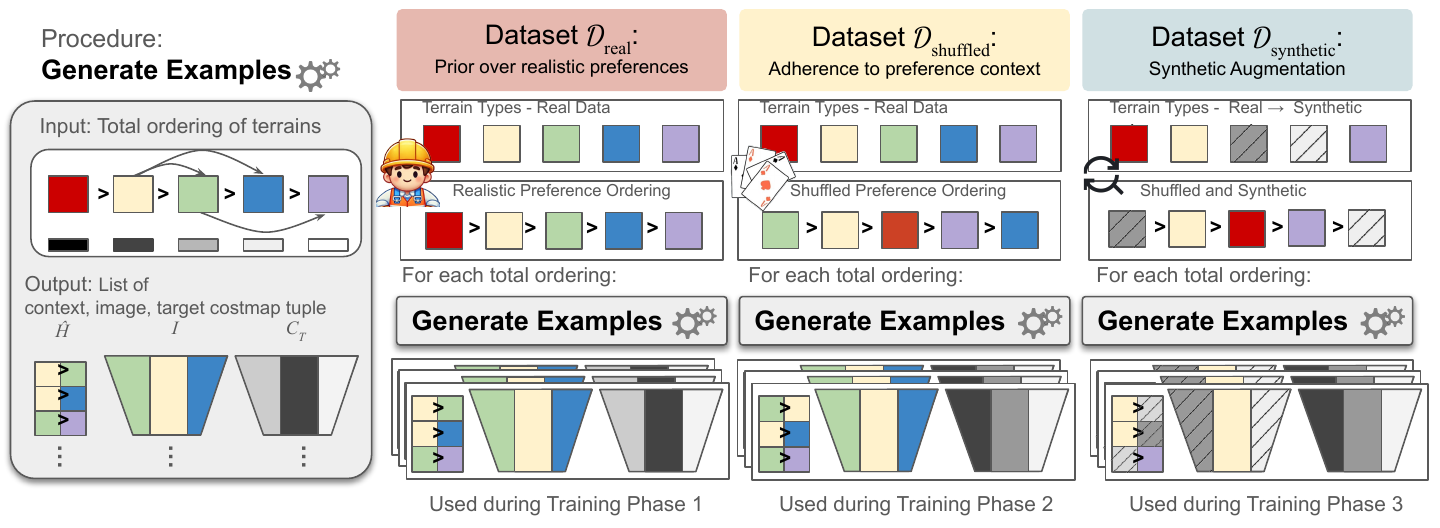}
        \caption{Overview of the dataset structure. Each training example contains a preference context, image, and target costmap. We vary the preferences and images, resulting in a large combinatorial dataset despite the relatively small amount of real recorded data. In a later training phase, we also augment with synthetic data by artificially finding and replacing certain terrain types with synthetic terrain textures. The real-valued costs assigned to terrain types based on an input total ordering are shown in the Generate Examples procedure, where black represents low cost and white is high cost.}
        \label{fig:dataset}
        \vspace{-0.6cm}
    \end{figure*}

We now describe the dataset curation and training process for \textsc{pacer}.
\textsc{pacer} is trained using three distinct phases of supervised machine learning, each corresponding to a unique training dataset that corresponds to one of the desiderata described in Section \ref{sec:introduction}.
In what follows, we will first describe how we generate training examples, then describe each of the three training phases and the training procedure.

\vspace{0.25cm}
\subsection{Training Example Generation}
The datasets $\mathcal{D}$ we use to train the \textsc{pacer} model consist of tuples of preference contexts, images, and target costmaps $(\hat{H},I,C_T)$.
To construct these datasets, we bootstrap off of semantic terrain classification and use a pretrained terrain patch classifier that assigns one of $L$ predefined semantic labels to a given terrain image patch.

The inputs to the training example generation process are a single image $I$ along with a total ordering over terrain types $\tau_1 \succ \tau_2 \succ \ldots \succ \tau_L$, where each terrain type $\tau_l$ corresponds to a bank of image patches.
\textsc{pacer} assumes that the cost value associated with $\tau_l$ is given by $\humanFunc{}(\tau_l) = \frac{l-1}{L-1}$.
The bank corresponding to $\tau_l$ consists of patches of that type extracted from images collected during robot deployment.

We use these inputs to generate $\hat{\humanFunc{}}$ and $\costmap{}_T$.
To generate $\hat{\humanFunc{}}$, we first choose $n$ ordered pairs from the total ordering over the $L$ terrain types without replacement.
For each of the resulting ordered pairs, we sample uniformly at random patches from the corresponding patch banks, and use these $2n$ patches to construct $\hat{\humanFunc{}}$ according to the process detailed in Section \ref{sec:preference_context} above.
To generate $\costmap{}_T$, we perform semantic segmentation on $I$ and transform the segmented image into $C_T$ by setting the cost for a pixel labelled $l$ to be $\humanFunc{}(\tau_l)$.

Constructing training examples in this way encourages $\pacerFunction{}_{\phi*}$ to follow our necessary conditions.
First, because $\costmap{}_T$ assigns the same cost value to image locations that received the same semantic label, $\pacerFunction{}_{\phi*}$ is encouraged to identify regions of visually-similar terrain and assign \emph{equivalent} costs within the region, as per condition \eqref{eq:first}.
Second, because both $\hat{\humanFunc{}}$ and $\costmap_T$ are, by construction, consistent with $\humanFunc{}$, $\pacerFunction{}_{\phi*}$ is encouraged to predict costmaps given $\hat{\humanFunc{}}$ which preserve the partial ordering of $\humanFunc{}$, as per condition \eqref{eq:second}. Interestingly, assuming sequential segmented images are temporally consistent, we observe that $\pacerFunction{}_{\phi*}$ is encouraged to be viewpoint-invariant.

During inference time, there are no semantic labels and only the visual appearances of terrains are considered.

\subsection{Dataset Size}

The size of the space from which we sample data is very large. From a total ordering of $L$ \edit{discrete terrain types}, there are $m = \binom{L}{2}$ different ordered pairs of terrains and $\binom{m}{n}$ different sets of $n$ pairs. For each set of $n$ pairs, there $n!$ ways to shuffle the pairs to construct the preference context, yielding $\binom{m}{n}\cdot n!$ possible preference contexts. Moreover, for each terrain type in the preference context, we sample a patch from the bank.  Our dataset contains a bank of around $800$ patches for each terrain label, and about $950$ full images. Therefore, for each total ordering, we have $\binom{\binom{L}{2}}{n}n!$ arrangements of labels into preference contexts, where we sample a patch from a bank of $800$ patches for each terrain type. For $L$ \edit{terrains}, $n = L \log L$ pairs are needed to describe a total ordering, though we evaluate on a smaller $n=3$ pairs due to size considerations for the model and dataset.

\vspace{0.25cm}
\subsection{Training Phases}
Each of the three system desiderata stated in Section \ref{sec:introduction} is manifested in a distinct training phase, each of which utilizes a unique training dataset generated using the procedure described above.
More specifically, these phases generate datasets $\mathcal D_{\mathrm{real}}$, $\mathcal D_{\mathrm{shuffled}}$, and $\mathcal D_{\mathrm{synthetic}}$, which promote adherence to prior preferences in seen terrains, robustness to new preferences, and robustness to new terrains, respectively.

A visualization of each of these phases is given in Figure \ref{fig:dataset}, and we describe each phase in more detail below.

\noindent\textbf{Training Phase 1: Pretraining with Real Data and Realistic Preferences}. To promote a prior towards an overall ``realistic" ordering (as per our first desired property), \textsc{pacer}'s first training phase constructs and utilizes a dataset $\mathcal D_{\mathrm{real}}$ generated using real-world data collected from robot deployments around our campus at The University of Texas at Austin and realistic preferences over terrain classes.
An example of a realistic preference ordering is as follows: $\texttt{concrete}$ \( \succ \) $\texttt{pebble}$ \( \succ \) $\texttt{grass}$ \( \succ \) $\texttt{marble}$ \( \succ \) $\texttt{bush}$. The ``realistic preferences" were defined by the first author according to considerations for robot safety (e.g. preferring grass over loose marble for a wheeled robot) and societal norms (e.g. preferring concrete over grass to avoid trampling lawns, even though both terrains are relatively safe).

\noindent\textbf{Training Phase 2: Augmentation with Changed Preferences}. During deployment in terrains not seen during training, the robot should adhere to preferences given by the operator (as per our second desired property).
Even when operator preferences contradict “realistic preferences”, the robot should follow operator preferences over learned priors.
To encourage this adherence to the ordering in the preference context, we train using the same real data but with changed preferences on a smaller corpus of data by using a randomly-permuted total ordering over terrain labels.

\noindent\textbf{Training Phase 3: Augmentation with Synthetic Terrains}. To promote the model's ability to generalize to terrains unseen during training (as per our third desired property), we further train with synthetically augmented data.

We pick a random subset of terrains to replace and randomly permute the preference order.
An image containing at least one such terrain is selected, and those terrains are artificially replaced with terrain textures from an open-source database \cite{polyhaven} using dense segmentation. \edit{We used 14 synthetic textures}. The training example is formed with a preference context (where terrains have been replaced), the image, and a costmap with costs reassigned according to the new preference order.

\edit{In the first phase, only $\mathcal D_{\mathrm{real}}$ is used for training examples. In the second, both $\mathcal D_{\mathrm{real}}$ and $\mathcal D_{\mathrm{shuffled}}$ are used. In the third, all three datasets are used. Within a phase, training examples are drawn uniformly among the datasets used}. Before training on $\mathcal D_{\mathrm{real}}$, weights are initialized randomly. After completing a training phase, we switch to the next \edit{phase} starting from the previous trained weights. \edit{We trained for 100 epochs in phase 1, 5 in phase 2, and 100 in phase 3.}

\section{Experiments}
\label{sec:experiments}



To evaluate \textsc{pacer}, we seek to answer the following questions empirically:

\begin{enumerate} [label=\bfseries\arabic*]
    \item How effectively is the robot able to navigate in terrains \textit{seen during training} when the preference context contains \textbf{(a)} only seen terrains or \textbf{(b)} only previously unseen terrains?
    \item How effectively is the robot able to navigate in \textit{unseen} terrains when the preference context contains \textbf{(a)} only those unseen terrains or \textbf{(b)} only seen terrains?
\end{enumerate}

By dividing deployment scenarios into the four situations above, we will be able to understand the performance of \textsc{pacer} under the four ways to combine \textit{seen} and \textit{unseen} terrains in the preference context and environment.
In \textbf{1b} and \textbf{2b}, the preference context does not provide information about the terrains appearing in the environment, so the robot must rely on learned priors about realistic cost assignment. \edit{We term this scenario as having an \emph{uninformative context}.}

Evaluations are performed using simulated experiments on an aerial map. In a later section, we also provide results from real robot deployments. We compare against \textsc{sterling} \cite{karnan2023sterling} (a representation learning approach) and a classifier (a semantics-based approach) as baselines. \edit{The classifier is the same as used generate the training data for \textsc{pacer}.} The same model for \textsc{pacer} is utilized across all environments. No retraining or fine-tuning is done. Additionally, a single context is used for the duration of a simulated deployment in an environment (i.e. the context is not switched out midway through the path-planning).

To quantify the navigation performance in our experiments, we posit that factors such as the distance traversed or closeness to a human-defined trajectory do not matter as much as traveling on only the preferred terrains. We therefore assign terrain types in each experiment a low, medium, and high cost, \edit{according to the human deployer's preference} and report the proportion of the planned path which \edit{belonging to each tier.} \edit{The assignment of semantic terrain types to these three tiers is based on a hidden total ordering of all the terrains which appear in the environment. Each of the methods and baselines tested are deployed with a preference (provided according to their respective representations) consistent with the hidden total ordering.} 
Note that we have purposely chosen this metric to be different from the commonly-used Hausdorff distance between the planned trajectory and one defined by a human operator, which can vary greatly when there are multiple valid paths to the goal.

\begin{figure*}
    \centering
    \vspace{.3cm}
    \includegraphics[width=.9\linewidth, trim=0 0 0 0, clip]{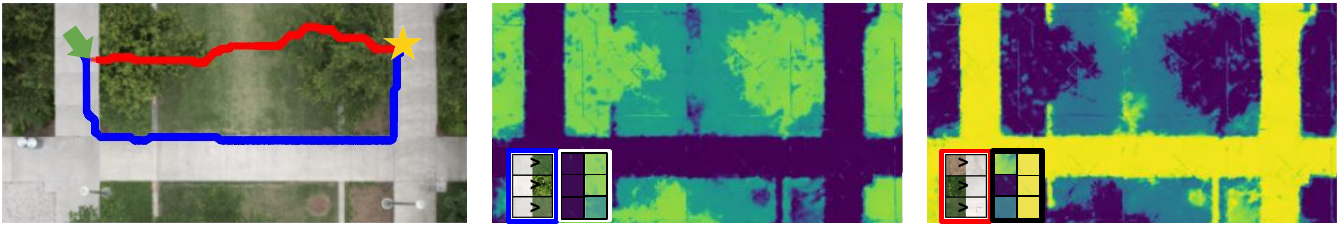}
    \caption{The effect of changing preference on path planning is shown on the left image of an aerial map. The blue path corresponds to the middle costmap and the red path corresponds to the right costmap. The generated costmaps reflect the preferences provided in the context. \edit{Dark purple corresponds to low cost and yellow is high cost.}}
    \label{fig:aerial_map_and_costmaps}
    \vspace{-.4cm}
\end{figure*}

\subsection{Aerial Map Experiments}
 In our simulated experiments, we build aerial maps from drone footage of three locations around our campus, which we consider seen environments. We also use open-source aerial maps \cite{openaerialmap} from around the world, covering a wide variety of both urban and natural terrain types and which we consider unseen. For each of the seen and unseen environments, we provide a start and goal location and test varying operator preferences. We test realistic preferences, and ``inverted" preferences, in which each of the pairwise orderings in the realistic preference are reversed. Given the robot’s pose on the aerial map, the robot’s projected bird’s eye view can be found, and used as input to generate the local costmap. Planning is done using the A* algorithm \cite{Hart1968} on the costmap. \edit{In Fig. \ref{fig:sim_paths}, we show examples of paths planned using \textsc{pacer} in our aerial map simulator. These experiments are an evaluation purely over planning based on terrain preference, with no kinodynamic constraints, no errors due to localization, and no costs associated with elevation.}

 \begin{figure}
        \centering

        \includegraphics[width=0.9\linewidth, trim=0 8 0 0, clip]{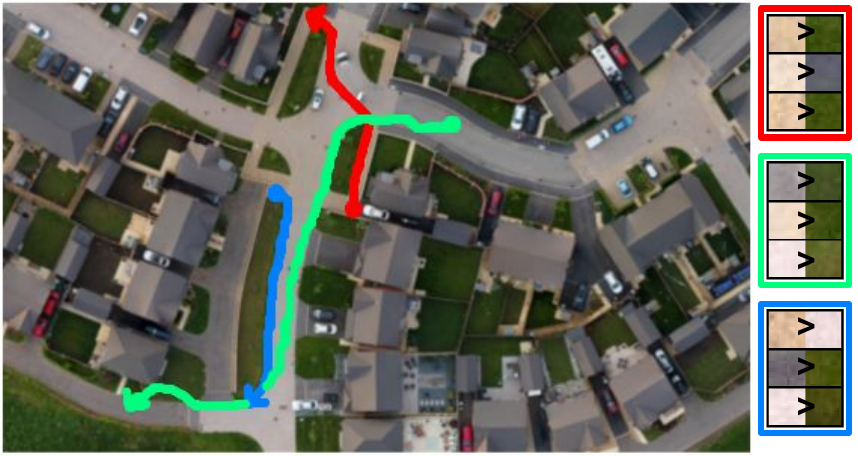}
        \caption{Examples of several paths planned using our method in an urban unseen environment and their corresponding preference contexts. Here, we visualize the large scale of these simulated deployments, and the diversity of visual terrain appearances. }
        \label{fig:sim_paths}
        \vspace{-.6cm}
    \end{figure}

 Tab. \ref{tab:seen_environment_experiments} displays the results for \textit{seen} environments. Towards answering question \textbf{1a}, \textsc{pacer} has similar results to the \textsc{sterling} baseline, as both approaches were trained on the same in-distribution data. When no useful context has been provided for \textsc{pacer} (i.e. the context contains only unseen terrains which do not match the environment as per \textbf{1b}), the results are on par with the classifier baseline. \textsc{pacer}'s success despite the lack of an informative preference context shows that the model has captured a prior over realistic cost-assignment for in-distribution terrains.

 Tab. \ref{tab:unseen_environment_experiments} displays the results for \textit{unseen} environments. Towards answering question \textbf{2a}, when given an informative context, \textsc{pacer} outperforms the \textsc{sterling} baseline. Though representation-learning approaches like \textsc{sterling} should theoretically generalize due to their continuous representation space, this is contingent on similar terrains forming clusters in this space, which may not be the case for unseen terrains. Fapt install r-base-coreor unseen environments, the classifier baseline has been omitted, as it allows no way for a user to provide terrain preferences for classes that are not predefined. \textsc{pacer} overcomes the limitations of previous paradigms, as it both allows preferences in unseen environments to be expressed and generalizes well to these unseen environments. Additionally, when no useful context has been provided for the unseen terrains (per \textbf{2b}), \textsc{pacer} is not able to adapt.

\begin{table}[tb]
\centering

\begin{minipage}{0.45\textwidth}
\centering
\resizebox{\textwidth}{!}{
\begin{tabular}{|l|c|c|c|}
\hline
\textbf{Method} & \textbf{Low (\%)} & \textbf{Medium (\%)} & \textbf{High (\%)} \\
\hline
\textsc{pacer} & \cellcolor{green!70}73.83\% & \cellcolor{yellow!32}19.70\% & \cellcolor{orange!10}6.47\% \\
\textsc{pacer} (uninformative context) & \cellcolor{green!50}67.97\% & \cellcolor{yellow!20}15.89\% & \cellcolor{orange!35}16.15\% \\
\textsc{sterling} & \cellcolor{green!75}74.72\% & \cellcolor{yellow!30}18.65\% & \cellcolor{orange!20}6.63\% \\
Classifier & \cellcolor{green!40}65.86\% & \cellcolor{yellow!50}21.46\% & \cellcolor{orange!30}12.69\% \\
\hline
Ground Truth & \cellcolor{green!90}82.82\% & \cellcolor{yellow!15}13.69\% & \cellcolor{orange!10}3.49\% \\
\hline
\end{tabular}
}
\caption{Proportion of planned paths that traverse low, medium, and high-cost terrains in seen environments, relating to \textbf{1a,1b}.
}
\label{tab:seen_environment_experiments}
\end{minipage}
\vspace{.5cm}

\begin{minipage}{0.45\textwidth}
\centering
\resizebox{\textwidth}{!}{
\begin{tabular}{|l|c|c|c|}
\hline
\textbf{Method} & \textbf{Low (\%)} & \textbf{Medium (\%)} & \textbf{High (\%)} \\
\hline
\textsc{pacer} & \cellcolor{green!80}81.85\% & \cellcolor{yellow!30}8.00\% & \cellcolor{orange!20}10.15\% \\
\textsc{pacer} (uninformative context) & \cellcolor{green!30}53.83\% & \cellcolor{yellow!10}2.20\% & \cellcolor{orange!70}43.97\% \\
\textsc{sterling} & \cellcolor{green!45}61.39\% & \cellcolor{yellow!40}10.12\% & \cellcolor{orange!50}26.86\% \\
\hline
Ground Truth & \cellcolor{green!95}91.85\% & \cellcolor{yellow!20}6.02\% & \cellcolor{orange!10}2.13\% \\
\hline
\end{tabular}
}
\caption{Proportion of planned paths that traverse low, medium, and high-cost terrains in unseen environments, relating to \textbf{2a,2b}.}
\label{tab:unseen_environment_experiments}
\end{minipage}
\vspace{-.5cm}
\end{table}

\begin{table}[tb]
\centering
\vspace{.3cm}

\begin{minipage}{0.45\textwidth}
\centering
\resizebox{\textwidth}{!}{
\begin{tabular}{|l|c|c|c|}
\hline
\textbf{In-Dist. Scenario} & \textbf{Low (\%)} & \textbf{Medium (\%)} & \textbf{High (\%)} \\
\hline
\textbf{Realistic Preference} & & & \\
$\mathcal{D}_\mathrm{real}$ & \cellcolor{green!95}95.97\% & \cellcolor{yellow!10}1.70\% & \cellcolor{orange!15}2.32\% \\
$\mathcal{D}_\mathrm{shuffled}$ & \cellcolor{green!60}70.50\% & \cellcolor{yellow!60}18.94\% & \cellcolor{orange!50}10.55\% \\
$\mathcal{D}_\mathrm{synthetic}$ & \cellcolor{green!70}76.66\% & \cellcolor{yellow!58}18.43\% & \cellcolor{orange!20}4.91\% \\
\hline
\textbf{Inverted Preference} & & & \\
$\mathcal{D}_\mathrm{real}$ & \cellcolor{green!15}\textcolor{blue}{18.55\%} & \cellcolor{yellow!50}\textcolor{blue}{20.97\%} & \cellcolor{orange!90}\textcolor{blue}{60.48\%} \\
$\mathcal{D}_\mathrm{shuffled}$ & \cellcolor{green!55}\textcolor{blue}{69.21\%} & \cellcolor{yellow!55}\textcolor{blue}{22.56\%} & \cellcolor{orange!30}\textcolor{blue}{8.23\%} \\
$\mathcal{D}_\mathrm{synthetic}$ & \cellcolor{green!60}\textcolor{blue}{70.22\%} & \cellcolor{yellow!53}\textcolor{blue}{21.32\%} & \cellcolor{orange!35}\textcolor{blue}{8.46\%} \\
\hline
\end{tabular}
}
\caption{Proportion of planned paths that traverse low, medium, and high-cost terrains in seen environments for each training phase. The model after each phase of training is tested using realistic and  inverted contexts.}
\label{tab:seen_ablation}
\end{minipage}
\vspace{.3 cm}

\begin{minipage}{0.45\textwidth}
\centering
\resizebox{\textwidth}{!}{
\begin{tabular}{|l|c|c|c|}
\hline
\textbf{Out-of-Dist. Scenario} & \textbf{Low (\%)} & \textbf{Medium (\%)} & \textbf{High (\%)} \\
\hline
\textbf{Realistic Preference} & & & \\
$\mathcal{D}_\mathrm{real}$ & \cellcolor{green!60}{60.78\%} & \cellcolor{yellow!20}7.05\% & \cellcolor{orange!70}32.18\% \\
$\mathcal{D}_\mathrm{shuffled}$ & \cellcolor{green!65}65.16\% & \cellcolor{yellow!15}6.44\% & \cellcolor{orange!65}28.40\% \\
$\mathcal{D}_\mathrm{synthetic}$ & \cellcolor{green!88}88.09\% & \cellcolor{yellow!5}2.57\% & \cellcolor{orange!25}9.34\% \\
\hline
\textbf{Inverted Preference} & & & \\
$\mathcal{D}_\mathrm{real}$ & \cellcolor{green!25}\textcolor{blue}{22.59\%} & \cellcolor{yellow!50}\textcolor{blue}{15.31\%} & \cellcolor{orange!90}\textcolor{blue}{62.10\%} \\
$\mathcal{D}_\mathrm{shuffled}$ & \cellcolor{green!40}\textcolor{blue}{39.86\%} & \cellcolor{yellow!45}\textcolor{blue}{13.96\%} & \cellcolor{orange!80}\textcolor{blue}{46.17\%} \\
$\mathcal{D}_\mathrm{synthetic}$ & \cellcolor{green!65}\textcolor{blue}{65.65\%} & \cellcolor{yellow!70}\textcolor{blue}{22.10\%} & \cellcolor{orange!30}\textcolor{blue}{12.25\%} \\
\hline
\end{tabular}
}
\caption{Proportion of planned paths that traverse low, medium, and high-cost terrains in unseen environments for each training phase. The model after each phase of training is tested using realistic and inverted contexts.}
\label{tab:unseen_ablation}
\vspace{-.8cm}
\end{minipage}

\end{table}

\subsection{Ablation Study}
To understand the effects of each phase in the training process, we perform an ablation study using the same environments as in the aerial simulator experiments. In the $\mathcal D_{\mathrm{real}}$ phase, the model is trained only on real data and realistic preferences. In the $\mathcal D_{\mathrm{shuffled}}$ phase, the model is pretrained on real data with realistic preferences, and then trained on a smaller amount of changed preferences. In $\mathcal D_{\mathrm{synthetic}}$ phase, the model is trained according to all three phases.

Results are shown in Tab. \ref{tab:seen_ablation}, \ref{tab:unseen_ablation}. In seen environments, the model trained on $\mathcal D_{\mathrm{real}}$ performs the best of the three with realistic preferences, but is unable to adapt when the preferences are inverted. \edit{Results for realistic preferences (black) are significantly better than those for inverted preference (blue) for this method in Tab. \ref{tab:seen_ablation}}.
In unseen environments \edit{(Tab. \ref{tab:unseen_ablation})}, this method performs the worst regardless of preference context. The model trained on $\mathcal D_{\mathrm{shuffled}}$ is able to adapt to changing preference orderings in seen environments, \edit{seen by comparing results for realistic preferences (black) to inverted preferences (blue) in Tab. \ref{tab:seen_ablation}}. However, it is unable to recognize and match new terrains in the preference context to the new environment \edit{as evidenced by the drop in performance from Tab. \ref{tab:seen_ablation} to Tab. \ref{tab:unseen_ablation}}. The model trained on  $\mathcal D_{\mathrm{synthetic}}$ is shown to both respect changing preference order and recognize new terrains. \edit{Our findings indicate that training the model to respond to preferences reduces performance with realistic preferences in seen environments, as $D_{\mathrm{shuffled}}$ and $D_{\mathrm{synthetic}}$ have slightly poorer performance than $D_{\mathrm{real}}$ with realistic but far better performance with inverted preferences. Furthermore, performance across all three methods decreases from seen to unseen since the scenario is harder, but $D_{\mathrm{synthetic}}$ has the least drop in performance. Variable preferences and unseen environments introduce greater complexity, requiring models to adapt and generalize. We prioritize these harder, practical scenarios, accepting a trade-off with peak accuracy in simpler settings.}

\edit{\subsection{Performance on Test Set}
There are infinite costmaps which may satisfy a preference over terrains (given as a partial order of terrains), since the scale of the numerical costs do not matter as long as the ordering is followed. Therefore, rather that direct comparison with a ``ground-truth" costmap, we report the \textit{margin ranking error} of predicted costmaps using holdout set from each dataset in Tab. \ref{tab:margin_ranking_error_table}. We sample the 500 points across the costmap, then for each pair of points, calculate the error as 0 if the relative magnitude of the two costs is correct. Otherwise the error is the difference between the two costs. We report the mean error for each category.}

\edit{For each dataset, the model which was trained on that dataset had the lowest error. Across all datasets, $\mathcal D_{\mathrm{synthetic}}$ had the lowest or second-lowest error, and had therefore the overall best performance.}





\begin{table}
\begin{minipage}{0.45\textwidth}
\vspace{.3cm}
\centering
\resizebox{.7\textwidth}{!}{
\begin{tabular}{|l|c|c|c|}
\hline
\textbf{Test Dataset} & {$\mathcal D_{\mathrm{real}}$} & {$\mathcal D_{\mathrm{shuffled}}$} & {$\mathcal D_{\mathrm{synthetic}}$} \\
\hline
\textbf{Model} & & & \\
{$\mathcal D_{\mathrm{real}}$} & \cellcolor{white}0.017 & \cellcolor{white}0.099 & \cellcolor{white}0.072 \\
{$\mathcal D_{\mathrm{shuffled}}$} & \cellcolor{white}0.086 & \cellcolor{white}0.019 & \cellcolor{white}0.064 \\
{$\mathcal D_{\mathrm{synthetic}}$} & \cellcolor{white}0.053 & \cellcolor{white}0.069 & \cellcolor{white}0.013 \\
\hline
\end{tabular}
}
\caption{Margin ranking error for each of the three models on the test set of each of the three datasets. Lower error is better.}
\label{tab:margin_ranking_error_table}
\end{minipage}
\vspace{-.6cm}
\end{table}

\subsection{Discussion}

    The results presented in this section demonstrate that \textsc{pacer} fulfills our three design desiderata. The adherence of \textsc{pacer} to realistic preferences when not given an informative preference context fulfills the first key property of being able to make inferences when there is no context by capturing a prior over realistic preferences (\textbf{1b}). As per the second key property, \textsc{pacer} has been shown to align costs to preferred terrains even as preferences are varied (\textbf{1a}). An example of paths planned according to different preferences is shown in Fig. \ref{fig:aerial_map_and_costmaps}. The performance of \textsc{pacer} in both seen and unseen environments when  given an informative preference context shows that \textsc{pacer} exhibits the third key property (\textbf{1a}, 
    \textbf{2a}).

\section{Real Robot Experiments}
We now seek to demonstrate that \textsc{pacer} performs well during execution in the real world. We deploy our method, \textsc{sterling}, and the classifier approach on a mobile robot at four different locations on the UT campus which are not included in the simulated environments. These four locations cover \texttt{red brick}, \texttt{concrete sidewalk}, \texttt{grass}, \texttt{mulch}, and \texttt{pebble pavement}.

Since all methods are trained to be view-point invariant and platform-agnostic, we trained them all with the same data collected from a Boston Dynamics Spot, and deployed zero-shot on a Clearpath Jackal which has significant differences in viewpoint and mobility than the Spot. We evaluate the performance of each method with a realistic preference on a variety of terrain types. In these experiments, we measure the robot's ability to execute the plan, which includes robustness to a different viewpoint and platform. We integrate each method with a sampling-based local planner \cite{biswas2013amrl} and maintain the same planner parameters to ensure fairness.

Tab. \ref{tab:performance} shows results from real robot experiments. Our approach had the most successful trials across all environments. While the classifier performed well in environments 1 and 2, we hypothesize that difficult lightning conditions and variations in terrain appearance caused the failures in environments 3 and 4. Though \textsc{sterling} performed as the best baseline in the simulated experiments (which involved only planning), it seemed to be unable to execute these plans in real-robot experiments. Many of the failure cases involved the robot driving slightly off-path and just grazing the undesirable terrains. In patch-based representations, a single patch may contain multiple different terrains (e.g. half sidewalk and half grass), so the cost assigned to the patch would be some combination of the different terrain costs, resulting in a coarser degree of control on the physical robot. Our approach overcomes this limitation since it directly outputs a fine-grained costmap in a single forward pass.

\begin{figure}[t]
        \vspace{.3cm}
        \centering
        \includegraphics[width=\linewidth]{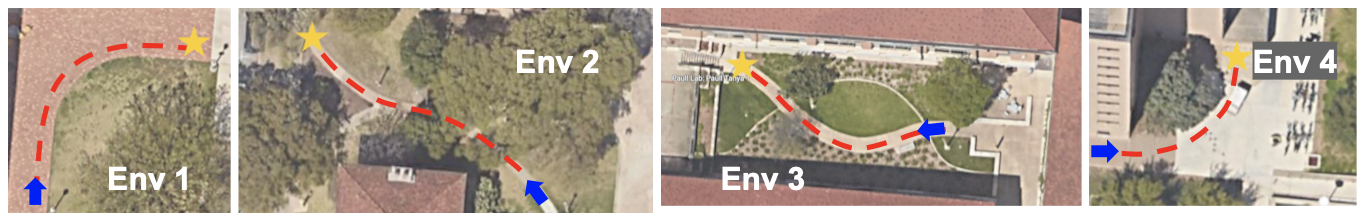}
        \caption{The four environments where real-robot trials were performed. Blue arrows and yellow stars show start and goal locations respectively. The red dashed line marks the intended path based on operator preference.}
        \label{fig:real_envs}
        \vspace{-.2cm}
    \end{figure}

\begin{table}[tb]
    \centering
    \begin{tabular}{>{\raggedright\arraybackslash}p{2.5cm} cccc}
        \toprule
        \textbf{Environment} & \textbf{Env 1} & \textbf{Env 2} & \textbf{Env 3} & \textbf{Env 4} \\
        \midrule
        \textbf{Classifier} & \textbf{5/5} & 4/5 & 1/5 & 2/5 \\
        \textbf{\textsc{sterling}}   & 0/5 & 2/5 & 0/5 & \textbf{4/5} \\
        \textbf{Ours}       & \textbf{5/5} & \textbf{5/5} & \textbf{5/5} & \textbf{4/5} \\
        \bottomrule
    \end{tabular}
    \caption{Number of successes per 5 trials of different approaches across various environments. A trial is a success if the robot reaches the goal without traversing across undesirable terrain and without operator intervention.}
    \label{tab:performance}
    \vspace{-.5cm}
\end{table}


\section{Limitations and Future Work}
 \label{sec:limitations}
 The experiments reported in this paper were conducted on a single campus with one robot. Future work should extend the study to more varied terrains and deploy multiple robots with diverse sensors. Additionally, we aim to evaluate the ease with which human users can express preferences using our method, as well as explore alternative mechanisms for preference expression in costmap generation. \edit{Incorporating a variable-length preference context would also be an interesting direction for future work, enabling greater flexibility and adaptability in representing user preferences. Exploring other model architectures could further enhance the system's robustness and performance in diverse scenarios.} Finally, incorporating depth data into the costmap generation is another key area for improvement, as our current reliance on a homography transformation assumes flat ground, leading to limitations like the inability to avoid obstacles such as concrete curbs, despite their preferable terrain.


\section{Conclusion}
\label{sec:conclusion}

	In this paper we presented \textsc{pacer}, a novel architecture and training approach to quickly produce costmaps according to arbitrary user preferences and new terrains with no fine-tuning. Our approach was evaluated against semanics-based and representation-learning baselines in both simulated and real robot experiments. We have shown this approach to be highly adaptable to new preferences and terrains, as well as able to infer the traversability of some terrains according to realistic preferences.

\section{Acknowledgements}
\label{sec:acknowledgements}
{
This work has taken place in the Autonomous Mobile Robotics Laboratory (AMRL) at UT Austin. AMRL research is supported in part by NSF (CAREER-2046955, OIA-2219236, DGE-2125858, CCF-2319471), ARO (W911NF-23-2-0004), Amazon, and JP Morgan. Any opinions, findings, and conclusions expressed in this material are those of the authors and do not necessarily reflect the views of the sponsors.}

{A portion of this work has taken place in the Learning Agents Research Group (LARG) at UT Austin.  LARG research is supported in part by NSF (FAIN-2019844, NRT-2125858), ONR (N00014-18-2243), ARO (W911NF-23-2-0004, W911NF-17-2-0181), Lockheed Martin, and UT Austin's Good Systems grand challenge.  Peter Stone serves as the Executive Director of Sony AI America and receives financial compensation for this work.  The terms of this arrangement have been reviewed and approved by the University of Texas at Austin in accordance with its policy on objectivity in research.}

\begingroup
\scriptsize
\bibliographystyle{IEEEtran}
\bibliography{citations}
\endgroup
\end{document}